\documentclass{article}

\usepackage[final]{neurips_2022}

\usepackage[utf8]{inputenc} 
\usepackage[T1]{fontenc}    
\usepackage[hidelinks]{hyperref}       
\usepackage{url}            
\usepackage{booktabs}       
\usepackage{amsfonts}       
\usepackage{nicefrac}       
\usepackage{microtype}      
\usepackage{xcolor}         

\usepackage{amsmath}
\usepackage{amssymb}
\usepackage[ruled,linesnumbered]{algorithm2e}
\usepackage[noend]{algpseudocode}
\usepackage[capitalise]{cleveref}
\usepackage{graphicx}
\usepackage{subcaption}
\usepackage{ bbold }
\usepackage{enumitem}
\usepackage{natbib}
\usepackage{float}
\usepackage{cleveref}
\crefformat{equation}{(#2#1#3)}

\DeclareMathOperator*{\argmax}{arg\,max}

\newcommand{\norm}[1]{\|#1\|}
\newcommand{\cS}{\mathcal S}
\newcommand{\A}{\mathcal A}
\newcommand{\1}{\mathbb{1}}
\newcommand{\E}{\mathbb{E}}
\newcommand{\event}{\mathcal{E}}
\newcommand{\R}{\mathbb{R}}
\newcommand{\Prob}[0]{\mathbb{P}}

\newcommand{\q}{\tilde q}

\newcommand{\Topt}[1]{T_{#1}}

\newtheorem{lemma}{Lemma}
\newtheorem{corollary}{Corollary}
\newtheorem{theorem}{Theorem}
\newtheorem{definition}{Definition}
\newtheorem{assumption}{Assumption}

\newtheorem{innercustomgeneric}{\customgenericname}
\providecommand{\customgenericname}{}
\newcommand{\newcustomtheorem}[2]{%
  \newenvironment{#1}[1]
  {%
   \renewcommand\customgenericname{#2}%
   \renewcommand\theinnercustomgeneric{##1}%
   \innercustomgeneric
  }
  {\endinnercustomgeneric}
}
\newcustomtheorem{customthm}{Theorem}
\newcustomtheorem{customlemma}{Lemma}
\newcustomtheorem{customcorollary}{Corollary}

\newenvironment{proof}[1]{\par\noindent\underline{Proof:}\space#1}{\hfill $\blacksquare$}

\title{Doubly-Asynchronous Value Iteration: \\Making Value Iteration Asynchronous in Actions}

\author{
  Tian Tian \quad Kenny Young \quad Richard S. Sutton\\
  University of Alberta and Alberta Machine Intelligence Institute \\
  Edmonton, Alberta, Canada  \\
  \texttt{\{ttian, kjyoung, rsutton\}@ualberta.ca}\\
}

\begin{document}

\maketitle

\begin{abstract}
Value iteration (VI) is a foundational dynamic programming method, important for learning and planning in optimal control and reinforcement learning.  VI proceeds in batches, where the update to the value of each state must be completed before the next batch of updates can begin.  Completing a single batch is prohibitively expensive if the state space is large, rendering VI impractical for many applications.  Asynchronous VI helps to address the large state space problem by updating one state at a time, in-place and in an arbitrary order.  However, Asynchronous VI still requires a maximization over the entire action space, making it impractical for domains with large action space.  To address this issue, we propose \textit{doubly-asynchronous value iteration} (DAVI), a new algorithm that generalizes the idea of asynchrony from states to states and actions.  More concretely, DAVI maximizes over a sampled subset of actions that can be of any user-defined size.  This simple approach of using sampling to reduce computation maintains similarly appealing theoretical properties to VI without the need to wait for a full sweep through the entire action space in each update.  In this paper, we show DAVI converges to the optimal value function with probability one, converges at a near-geometric rate with probability $1-\delta$, and returns a near-optimal policy in computation time that nearly matches a previously established bound for VI.  We also empirically demonstrate DAVI's effectiveness in several experiments. 
\end{abstract}

\section{Introduction}
Dynamic programming has been used to solve many important real-world problems, including but not limited to wireless networks \citep{wireless1, wireless2, wireless3} resource allocation \citep{powell2002adaptive}, and inventory problems \citep{bensoussan2011dynamic}.  Value iteration (VI) is a foundational dynamic programming algorithm central to learning and planning in optimal control and reinforcement learning.  

VI is one of the most widely studied dynamic programming algorithm \citep{Williams93analysisof, neuro, mdp}. VI starts from an arbitrary value function and proceeds by updating the value of all states in batches.  The value estimate for each state in the state space $\cS$ must be computed before the next batch of updates will even begin:
\begin{align}
    v(s) \leftarrow \max_{a \in \A} \left \{r(s,a) + \gamma \sum_{s' \in \cS} p(s'|s,a) \bar v(s') \right\} \quad \text{for all } s \in \cS, \label{eq:vi}
\end{align}
where $s' \in \cS$ is the next state, $p(s'|s,a)$ is the transition probability, and $r(s,a)$ is the reward.  VI maintains two arrays of real values $v, \bar v$, both the size of state space $S \doteq |\cS|$, with $\bar v$ being used to make the update and $v$ being used to keep track of the updated values.  In domains with large state space, for example wireless networks where the state-space of the network scales exponentially in the number of nodes, computing a single batch of state value updates using VI is prohibitively expensive, rendering VI impractical.

An alternative to updating the state values in batches is to update the state values one at a time, using the most recent state value estimates in the computation. Asynchronous value iteration \citep{Williams93analysisof, neuro, mdp} starts from an arbitrary value function and proceeds with in-place updates:
\begin{align}
    v(s_n) \leftarrow \max_{a \in \A} \left \{r(s_n,a) + \gamma \sum_{s' \in \cS} p(s'|s_n,a) v(s') \right\}, \label{eq:avi}
\end{align}
where $s_n \in \cS$ is the sampled state for update in iteration $n$ with all other states $s \neq s_n$ value remain the same. 

Although Asynchronous VI helps to address domains with large state space, it still requires a sweep over the actions for every update due to the maximization operation over the look-ahead values: $L^v(s,a) \doteq r(s,a) + \gamma \sum_{s' \in \cS} p(s'|s,a) v(s')$ for a given $v \in \R^S$. Evaluating the look-ahead values, as large as the action space, can be prohibitively expensive in domains with large action space.  For example, running Asynchronous VI would be impractical in fleet management, where the action set is exponential in the size of the fleet.

There have been numerous works that address the large action space problem.  \cite{szepesvari1996generalized} proposed an algorithm called sampled-max, which performs a max operation over a smaller subset of look-ahead values.  Their algorithm resembles Q-learning, requiring a step-size parameter that needs additional tuning.  However, their algorithm does not converge to the optimal value function $v^*$. \cite{Williams93analysisof} presented convergence analysis on a class of asynchronous algorithms, including some that may help address the large action space problem.  \cite{Hubert2021LearningAP} and \cite{danihelka2022policy} explore ways to achieve policy improvement while sampling only a subset of actions in Monte Carlo Tree Search.  
\cite{zaheen} focused on the generation of a smaller candidate action set to be used in planning with continuous action space.   

We consider the setting where we have access to the underlying model of the environment and propose a variant of Asynchronous VI called \textit{doubly-asynchronous value iteration} (DAVI) that generalizes the idea of asynchrony from states to states and actions.  Like Asynchronous VI, DAVI samples a state for the update, eliminating the need to wait for all states to complete their update in batches.  Unlike Asynchronous VI, DAVI samples a subset of actions of any user-defined size via a predefined sampling strategy.  It then computes only the look-ahead values of the corresponding subset and a best-so-far action, and updates the state value to be the maximum over the computed values.  The intuition behind DAVI is the idea of incremental maximization, where maximizing over a few actions could improve the value estimate for a certain state, which helps to evaluate other state-action pairs in subsequent back-ups.  This simple approach of using sampling to reduce computation maintains similarly appealing theoretical properties to VI.  In particular, we show DAVI converges to $v^*$ with probability 1 and at a near-geometric rate with probability $1 - \delta$.  Additionally, DAVI returns an $\epsilon$-optimal policy with probability $1-\delta$ using
\begin{align}
    O\left(m S H_{\gamma, \epsilon} \left(\ln\left(\frac{SH_{\gamma, \epsilon}}{\delta}\right)\Big/\ln\left(\frac{1}{1-q_{min} }\right)\right)\right)
\end{align}
elementary operations, where $m$ is the size of the action subset, $H_{\gamma, \epsilon}$ is a horizon term, and $q_{min}$ is the minimum probability that any state-action pair is sampled. We also provide a computational complexity bound for Asynchronous VI, which to the best of our knowledge, has not been previously reported. Our computational complexity bounds for both DAVI and Asynchronous VI nearly match a previously established bound for VI \citep{vi_bound}.  Finally, we demonstrate DAVI's effectiveness in several experiments.

Related work by \cite{zeng2020asyncqvi} uses an incremental maximization mechanism that is similar to ours.  However, their work focuses on a different setting from ours, an asynchronous parallel setting, where an algorithm is hosted on several different machines running concurrently without waiting for synchronization of the computations.  Aside from this difference, their work considers the case where the agent has access to only a generative model \citep{kearns2002sparse} of the environment, whereas we assume full access to the transition dynamics and reward function.  They also provided a computational complexity bound, which differs significantly from ours due to differences in settings.  

\section{Background} \label{sec:background}
We consider a discounted Markov Decision Process $(\cS, \A, r, p, \gamma)$ consisting of a finite set of states $\cS$, a finite set of actions $\A$, a stationary reward function $r: \cS \times \A \to [0,1]$, a stationary transition function $p:\cS \times \A \times \cS \to [0,1]$, and a discount parameter $\gamma \in [0,1)$.  The stochastic transition between state $s \in \cS$ and the next state $s' \in \cS$ is the result of choosing an action $a \in \A$ in $s$.  For a particular $s,a \in \cS \times \A$, the probability of landing in various $s'$ is characterized by a transition probability, which we denote as $p(s'|s,a)$.  

For this paper, we use $\Pi$ to denote a set of deterministic Markov polices $\{ \pi: \cS \to \A \}$.  The value function of a state evaluated according to $\pi$ is defined as  $v_{\pi}(s) \doteq \E\left[ \sum_{n=0}^{\infty} \gamma^{n} r(s_n, a_n) | s_0 = s \right]$ for a $s \in \cS$.  The optimal value function for a state $s \in \cS$ is then $ v^*(s) \doteq \max_{\pi \in \Pi} v_{\pi}(s)$ and there exists a deterministic optimal policy $\pi^*$ for which the value function is $v^*$.  For the rest of the paper, we will consider $v$ as a vector of values of $\R^S$.  For a fixed $\epsilon > 0$, a policy $\pi$ is said to be $\epsilon$-optimal if $v_{\pi} \geq v^* - \epsilon \1$.  Finally, we will use $\norm{\cdot}$ to denote infinity norm (i.e, $\norm{v} = \max_i |v_i|$) and $A$ to denote the size of the action space.

\section{Making value iteration asynchronous in actions}
We assume access to the reward and transition probabilities of the environment.  In each iteration $n$, DAVI samples a state $s_n$ for the update and $m$ actions from the action set $\A$.  It then computes the corresponding look-ahead values of the sampled actions and the look-ahead value of a best-so-far action $\pi_n(s_n)$.  Finally, DAVI updates the state value to the maximum over the computed look-ahead values.  The size of the action subset can be any user-defined value between $1$ and $A$.  To maintain a best-so-far action, one for every state amounts to maintaining a deterministic policy in every iteration.  Recall $L^v(s,a) \doteq r(s, a) + \gamma \sum_{s' \in \cS} p(s'|s,a)v(s')$ for a given $s,a \in \cS \times \A$ and $v \in \R^S$, the pseudo-code for DAVI is shown in \cref{alg:davi}.

\begin{algorithm}[h!] 
    \SetAlgoLined
    \DontPrintSemicolon
 	\KwIn{State sampling distribution $p \in \Delta(S)$}
	\KwIn{A potentially state conditional distribution over the sets of actions of size $m$ denoted by $q$}
	\KwIn{Number of iterations $\tau$, see \cref{cor:complexity} for how to choose $\tau$ to obtain an $\epsilon$-optimal policy with high probability}
	Initialize the value function $v_0 \in \R^{S}$ \; 
    Initialize the policy $\pi_0$ to an arbitrary deterministic policy \;
    \For{$n = 0,\cdots,\tau$}{ 
        Sample a state from $p$ \;
        Sample $m$ actions from $q(\cdot | s_n)$ and let it be $\A_n$ \;
        Let $a^*_n = \argmax_{a \in A_n} L^{v_n}(s,a)$ with ties broken randomly \;
        Value back-up step: \;
        $v_{n+1}(s) = 
        \begin{cases}
            \max\left \{L^{v_n}(s,a_n^*), L^{v_n}(s, \pi_n(s)) \right\}  & \text{if } s = s_n 
            \\
            v_n(s) & \text{otherwise}
        \end{cases}$ \;
        Policy improvement step:\;
        $\pi_{n+1}(s) = 
        \begin{cases}
            a_n^* & \text{if } s = s_n \text{ and } L^{v_n}(s,a_n^*) > L^{v_n}(s,\pi_n(s))   \\
            \pi_{n}(s) & \text{otherwise}
        \end{cases}$ \;
    }
    \Return{$v_n, \pi_n$}
    \caption{DAVI($m, p,q, \tau$)} \label{alg:davi}
\end{algorithm}
Note that the policy will only change if there is another action in the newly sampled subset whose look-ahead value is strictly better than the current best-so-far action's look-ahead value.

\section{Convergence}\label{sec:conv}
We show in \cref{thm:conv} that DAVI converges to the optimal value function despite only maximizing over a subset of actions in each update.  Before showing the proof, we establish some necessary definitions, lemmas, and assumptions.  

\begin{definition}[$q_{min}$ and $p_{min}$] \label{def:1}
    Recall $p$ is a distribution over states and $q$ is a potentially state conditional distribution over the sets of actions of size $m$.  Then, we will use $\q(s,a)$ to denote the joint probability that a single state $s \in \cS$ is sampled for update with a particular action $a$ included in the set sampled by $q$.  Furthermore, let $q_{min} \doteq \min_{s,a} \q(s,a)$ and $p_{min} \doteq \min_{s} p(s)$.
\end{definition}

\begin{definition}[Bellman optimality operator and policy evaluation] \label{def:bell}
    Let $T: \R^S \to \R^S$.  For all $s \in \cS$, define $Tv(s) \doteq \max_{a \in \A} L^v(s,a)$.
    Let $T_{\pi}: \R^S \to \R^S$.  For a given $\pi$, for all $s \in \cS$, define $T_{\pi}v(s) \doteq L^{v}(s, \pi(s))$.
\end{definition}

\begin{definition}[DAVI back-up operator $T_n$] \label{def:davi_op}
Let $T_{n}: \R^S \to \R^S$. For a given $\A_n \sim \q$, $\pi_n \in \Pi$, $s_n \in \cS$, and for all $s \in \cS$ and $v \in \R^S$, define
\begin{align}
T_{n}v(s) \doteq
\begin{cases}
    \max_{a \in \A_n \cup \pi_n(s)} L^v(s,a) & \text{for } s = s_n \\
    v(s) & \text{otherwise}.
\end{cases}
\end{align}
\end{definition}

We show in \cref{appendix:davi} that $T_n$ is a monotone operator. 

\begin{assumption}[Initialization] \label{asumpt:init}
We consider the following initialisations, (i) $v_0 = 0 \1$, or (ii) $v_0 = -c \1$ for $c > 0$, or (iii) $v_0(s) \leq L^{v_0}(s,\pi_0(s))$ for all $s \in \cS$. 
\end{assumption}

\begin{lemma}[Monotonicity] \label{lemma:monotone}
   If DAVI is initialized according to (i),(ii), or (iii) of \cref{asumpt:init}, the value iterates of DAVI, $(v_n)_{n\geq 0}$ is a monotonically increasing sequence: $v_n \leq v_{n+1}$ for all $n \in \mathbb{N}_{0}$, if $r(s,a) \in [0,1]$ for any $s,a \in \cS \times \A$.  
\end{lemma}
\begin{proof}
See \cref{appendix:davi}.
\end{proof}

\begin{lemma}[Boundedness \citep{Williams93analysisof}] \label{lemma:bounded}
    Let $v_{\max} = \max_{s} v_0(s)$ and $v_{\min} = \min_{s} v_0(s)$ and recall that any reward $\in [0,1]$.  If we start with any $(v_0, \pi_0)$, then applying DAVI's operation on the $(v_0, \pi_0)$ thereafter, will satisfy: $
        \min\left\{0, v_{\min} \right\} \leq V_n(s) \leq \max\left\{\frac{1}{1-\gamma}, v_{\max} \right\}$,
    for all $s \in \cS$ and for all $n \in \mathbb{N_0}$.
\end{lemma}

\begin{lemma}[Fixed-point iteration \citep{szepesvari2010algorithms}] \label{lemma:fixpoint}
 
Given any $v \in \R^S$, and $T, T_{\pi}$ defined in \cref{def:bell}
    \begin{enumerate}
        \item $v_{\pi} = \lim_{n \to \infty} T_{\pi}^n v$ for a given policy $\pi$.  In particular for any $n \geq 0$, $\norm{T^n_{\pi} v - v_{\pi}} \leq \gamma^n \norm{v - v_{\pi}}$ where $v_{\pi}$ is the unique function that satisfies $T_{\pi} v_{\pi} = v_{\pi}$.
        \item $v^* = \lim_{n \to \infty} T^n v$ and in particular for any $n \geq 0$, $\norm{v^* - T^n v} \leq \gamma^n \norm{v^* - v}$, where $v^*$ is the unique function that satisfies $T v^* = v^*$.
    \end{enumerate}
\end{lemma}

\begin{theorem}[Convergence of DAVI] \label{thm:conv}
   Assume that $\q(s,a) > 0$ and $r(s,a) \in [0,1]$ for any $s,a \in \cS \times \A$, then DAVI converges to the optimal value function with probability 1, if DAVI initializes according to (i),(ii), or (iii) of \cref{asumpt:init}.
\end{theorem}

\begin{proof}  

By the Monotonicity \cref{lemma:monotone} and Boundedness \cref{lemma:bounded}, DAVI's value iterates $(v_n)_{n \geq 0}$ are a bounded and monotonically increasing sequence.  By the monotone convergence theorem, $\lim_{n \to \infty} v_n = \sup_{n} v_n \doteq \bar v$.  It remains to show that $\bar v = v^*$. We first show $ \bar v \leq v^*$ and then show $\bar v \geq v^*$ to conclude that $\bar v = v^*$.  We note that $v_n = T_{n-1} v_{n-1} \leq T v_{n-1}$, where $T$ is the Bellman optimality operator that satisfies the Fixed-point \cref{lemma:fixpoint} (2).   By the monotonicity of $T$ and the monontonicity of $v_n$'s, for any $n \geq 0$, $v_n \leq T^n v_0$.  By taking the limit of $n \to \infty$ on both sides, we get $\bar v \leq v^*$.  

Now, we show $\bar v \geq v^*$.  Let $(n_{k})_{k=0}^{\infty}$ be a sequence of increasing indices, where $n_0 = 0$, such that between $n_k$-th and $n_{k+1}$-th iteration, all state $s \in \cS$ have been updated at least once with an action set containing $\pi^*(s)$.  We note that the number of iterations between any $n_k$ and $n_{k+1}$ is finite with probability 1 since there is a finite number of states and actions, and all state-action pairs are sampled with non-zero probability.  Then, for any state $s$, let $t(s,k)$ be an iteration index such that $s_{t(s,k)} = s$,  $n_{k} \leq t(s, k) \leq n_{k+1}$, and $\pi^*(s) \in \A_{t(s,k)}$, then
\begin{align}
   &v_{n_{k+1}}(s) \geq v_{t(s,k) + 1}(s) \hspace{1cm} \text{by monotonicity of } v_n \\
   &= \max \left \{\max_{a \in \{\A_{t(s,k)} \cup \pi_{t(s,k)}\} \setminus \pi^*(s)} L^{v_{t(s,k)}}(s,a), L^{v_{t(s,k)}}(s, \pi^*(s)) \right \} \geq L^{v_{t(s,k)}}(s, \pi^*(s)) \\
   &= T_{\pi^*} v_{t(s,k)}(s) \geq T_{\pi^* }v_{n_k}(s) \quad \text{by monotonicity of the operator}. \label{eq:22}
\end{align}
By the $n_{k+1}$-th iteration, $v_{n_{k+1}} \geq T_{\pi^*}v_{n_k}$, where $T_{\pi^*}$ is the policy evaluation operator that satisfies the Fixed-point \cref{lemma:fixpoint} (1).  Continuing with the same reasoning, $v_{n_k} \geq T_{\pi^*}^k v_0$ for any $k \geq 0$.  By taking limit of $k \to \infty$ on both sides, we get $\bar v \geq v^*$.  Altogether, $\bar v = v^*$.
\end{proof}

\bf{Remark 1: } \rm The initialization requirement in the Convergence of DAVI \cref{thm:conv} can be relaxed to be any initialization, and DAVI will still converge to $v^*$ with probability 1.  A more general proof can be found in \cref{appendix:davi}, which follows a similar argument to that of the proof for Theorem 4.3.2 of \cite{Williams93analysisof}.  Intuitively, there exists a finite sequence of value back-up and policy improvement operations that will lead to one contraction, and if there are $l \in \mathbb{N}$ copies of such a sequence, this will lead to $l$ contractions.  Once the value iterates contract into an ``optimality-capture'' region, where all the policies $\pi_n$ are optimal thereafter, DAVI is performing policy evaluations of an optimal policy.  As long as all states are sampled infinitely often, the value iterates must converge to $v^*$.  Finally, we show that such a finite sequence as a contiguous subsequence exists in an infinite sequence of operators generated by a stochastic process.  

\bf{Remark 2: } \rm DAVI could be considered an Asynchronous Policy Iteration algorithm \citep{neuro}  since DAVI consists of a policy improvement step and a policy evaluation step.  However, the algorithmic construct discussed by \cite{neuro} does not exactly match that of DAVI with sampled action subsets.  Consequently, we could not directly apply Proposition 2.5 of \cite{neuro} to show DAVI's convergence.  A more useful analysis is that of \cite{Williams93analysisof}; we could have applied their Theorem 4.2.6 to show DAVI's convergence after having shown that DAVI's value iterates are monotonically increasing in \cref{lemma:monotone}.  However, \cite{Williams93analysisof} provide no convergence rate or computational complexity.  Therefore, we chose to present a different convergence proof, more closely related to the convergence rate proof in the next section.

\section{Convergence rate} \label{sec:conv_rate}
DAVI relies on sampling to reduce computational complexity, which introduces additional errors.  Despite this, we show in \cref{thm:conv_rate} that DAVI converges at a near-geometric rate and nearly matches the computational complexity of VI.

\begin{theorem}[Convergence rate of DAVI] \label{thm:conv_rate}
   Assume $\q(s,a) > 0$ and $r(s,a) \in [0,1]$ for any $s, a \in \cS \times \A$, and also assume DAVI initialises according to (i), (ii), (iii) of \cref{asumpt:init}.  With $\gamma \in [0,1)$ and probability $1-\delta$, the iterates of DAVI, $(v_n)_{n\geq 0}$ converges to $v^*$ at a near-geometric rate. In particular, with probability $1-\delta$, for a given $l \in \mathbb{N}$,
   \begin{align}
       \norm{v^* - v_n} \leq \gamma^l \norm{v^* - v_0},
   \end{align}
   for any n satisfying
   \begin{align}
   n \geq l \left \lceil  \ln \left( \frac{S l}{\delta} \right) \Big / \ln \left(\frac{1}{ 1 - q_{min}}\right) \right \rceil   \label{eq:n},
   \end{align}
   where $q_{min} = \min_{s,a} \q(s,a)$.
\end{theorem}
\begin{proof}
Recall from \cref{lemma:monotone}, we have shown $v_n \to v^*$ monotonically from below.  From \cref{thm:conv}, we have also defined
$(n_{k})_{k=0}^{\infty}$ to be a sequence of increasing indices, where $n_0 = 0$, such that  between the $n_k$-th and $n_{k+1}$-th iteration, all state $s \in \cS$ have been updated at least once with an action set containing $\pi^*(s)$.  At the $n_{k+1}$-th iteration, $v_{n_{k+1}} \geq T_{\pi^*}v_{n_k}$.  This implies that at the $n_{k+1}$-th iteration, DAVI would have $\gamma$-contracted at least once:
   \begin{align}
0 \leq v^* - v_{n_{k+1}}  &\leq v^* - T_{\pi^*}v_{n_k}, \implies 
\norm{v^* - v_{n_{k+1}}} \leq \norm{v^* - T_{\pi^*}v_{n_k}},    \\
\norm{v^* - T_{\pi^*}v_{n_k}} &= \norm{T_{\pi^*}v^* - T_{\pi^*}v_{n_k})} \leq \gamma \norm{v^* - v_{n_k}} \\
\implies \norm{v^* - v_{n_{k+1}}} &\leq \gamma \norm{v^* - v_{n_k}}. 
\end{align}

Consider dividing $n \in\mathbb{N}$ iterations into uniform intervals of length $N$ such that the $i$-th interval is $(iN, (i+1)N-1)$. Let $\event_i(s)$ denote the event that at some iteration in the $i$-th interval, state $s$ has been updated with an action set containing $\pi^*(s)$.  Therefore, an occurrence of event $\event_i \doteq \cap_{s \in \cS} \event_i(s)$ would mean that at $(i+1)N$-th iteration, $v_{(i+1)N}$ would have contracted at least once. Then, on the event $\event \doteq \cap_{i} \event_i = \cap_i \cap_{s \in \cS} \event_i(s)$, there have been at least $l$ $\gamma$-contraction after $n$ iterations.

We would like $\Prob(\event) \geq 1- \delta$ or alternatively the probability of failure event $\Prob(\event^c) \leq \delta$, for some $\delta > 0$.   However, just how large should $N$ be in order to maintain a failure probability of $\delta$? To answer this question, we first bound $\Prob(\event^c)$ using union bound:
\begin{align}
    \Prob(\event^c) = \Prob(\cup_{i} \cup_{s \in \cS} \event_i^c(s)) \leq \sum_{i = 1}^{l} \sum_{s \in \cS} \Prob(\event_i^c).\label{eq:37}
\end{align}
From \cref{def:1}, $\q(s, \pi^*(s))$ is the joint probability that a single state $s$ is sampled for update with $\pi^*(s)$ included in the action subset sampled by $q$.  Then, the probability that state $s$ is not updated with an action set containing $\pi^*(s)$ in $N$ iterations is $(1 - \q(s, \pi^*(s)))^N$.  Continuing from \cref{eq:37}, 
\begin{align}
    \Prob(\event^c) &\leq \sum_{i=1}^l \sum_{s \in \cS} (1 - \q(s, \pi^*(s)))^N  \leq Sl(1-q_{min})^N, 
 \end{align}
where $q_{min} \doteq \min_{s,a} \q(s, a)$.  Now set $Sl (1-q_{min})^N \leq \delta$ and solve for $N$,
 \begin{align}
     N &\geq  \ln \left( \frac{\delta}{Sl} \right) \Big / \ln \left( 1 -q_{min}\right).
 \end{align}
Thus, with probability at least $1 - \delta$, within
 \begin{align}
     n = l \left\lceil \ln \left( \frac{Sl}{\delta} \right) \Big / \ln \left( \frac{1}{1 - q_{min}}\right) \right\rceil 
 \end{align}
 iterations DAVI will have $\gamma$-contracted at least $l$ times. 
\end{proof}

\bf{Remark: } \rm We note that $p,q$ could be non-stationary and potentially chosen adaptively based on current value estimates, which is an interesting direction for future work.

\begin{corollary}[Computational complexity of obtaining an $\epsilon$-optimal policy] \label{cor:complexity}
 Fix an $\epsilon \in (0, \norm{v^* - v_0})$, and assume DAVI initialises according to (i), (ii), or (iii) of \cref{asumpt:init}.  Define 
 \begin{align}
     H_{\gamma, \epsilon} \doteq \ln \left(\frac{\norm{v^* - v_0}}{\epsilon}\right)/{1-\gamma}
 \end{align}
 as a horizon term.  Then, DAVI runs for at least
 \begin{align}
     \tau = H_{\gamma, \epsilon} \left(\ln\left(\frac{S H_{\gamma, \epsilon}}{\delta} \right)/\ln\left( \frac{1}{1-q_{min}}\right)\right)
 \end{align}
 iterations, returns an $\epsilon$-optimal policy $\pi_n: v_{\pi_n} \geq v^* - \epsilon \1$ with probability at least $1-\delta$ using $O\left(mS \tau \right)$ elementary arithmetic and logical operations, where $m$ is the size of the action subset and $S$ is the size of the state space.  Note that $\norm{v^* - v_0}$ is unknown but it can be upper bounded by $\frac{1}{1-\gamma} + \norm{v_0}$ given rewards are in $[0,1]$.
 \end{corollary}
 \begin{proof}
See \cref{appendix:davi}.
\end{proof}

\bf{Remark: } \rm As a straightforward consequence of \cref{thm:conv} and \cref{cor:complexity}, we show in \cref{appendix:davi} (Corollary \cref{appendix:opt_policy_complexity}) that DAVI returns an optimal policy $\pi^*$ with probability $1-\delta$ within a number of computations that depends on the minimal value gap between the optimal action and the second-best action with respect to $v^*$.

We can compare the computational complexity bound for DAVI (the result of \cref{cor:complexity}) to similar bounds for Asynchronous VI and VI.  As far as we know, computational complexity bounds for Asynchronous VI have not been reported in the literature.  We followed similar argument to \cref{thm:conv_rate} and \cref{cor:complexity} to obtain the computational complexity bound for Asynchronous VI in \cref{appendix:avi}. 
\begin{table}
    \caption{Computational complexity of VI, Asynchronous VI, DAVI} \label{tab:comp_complexity}
    \centering
    \begin{tabular}{l l l}
        \toprule
        \cmidrule(r){1-2}
        Algorithms & Computational complexity & References\\
        \midrule
         VI & $O\left(A S^2 H_{\gamma, \frac{\epsilon(1-\gamma)}{2\gamma}}\right)$ & \cite{vi_bound} \\
         Asynchronous VI & $O\left(AS H_{\gamma, \epsilon} \frac{\ln\left( \frac{S H_{\gamma, \epsilon}}{\delta} \right)}{\ln\left( \frac{1}{1-p_{min}}\right)}\right)$ & This paper \\
         DAVI & $O\left(mS H_{\gamma, \epsilon} \frac{\ln\left( \frac{S H_{\gamma, \epsilon}}{\delta} \right)}{\ln\left(\frac{1}{1-q_{min}}\right)}\right)$ & This paper \\
         \bottomrule
    \end{tabular}
\end{table}

Recall that $q_{min} \doteq \min_{s,a} \q(s,a)$. Consider the case of uniform sampling of states and actions. Uniform sampling of the states results in a probability of $\frac{1}{S}$ of sampling a particular state, while uniform sampling of $m$ actions without replacement results in a probability of $\frac{m}{A}$ of including a particular action in the subset.  Altogether $q_{min} = \frac{m}{SA}$.  Uniform sampling of state and action subset is the best sampling strategy for the bound $O\left(mS H_{\gamma, \epsilon} \left(\ln\left( \frac{S H_{\gamma, \epsilon}}{\delta} \right)/\ln\left( \frac{1}{1-q_{min}}\right)\right)\right)$ because any non-uniform strategy would result in $\q_{\min} < \frac{m}{SA}$.  Suppose $\q_{\min} = \frac{m}{SA}$, then
$ -m/\left(\ln(1- \frac{m}{SA})\right) \approx m/\left(\frac{m}{SA}\right) = SA$. Therefore, DAVI's computational complexity $O\left(mS H_{\gamma, \epsilon} \left(\ln\left( \frac{S H_{\gamma, \epsilon}}{\delta} \right)/\ln\left( \frac{1}{1-q_{min}}\right) \right)\right) \approx O \left(S^2 A H_{\gamma, \epsilon} \ln\left( \frac{S H_{\gamma, \epsilon}}{\delta} \right)\right)$.
Likewise, uniform sampling of state will result in $p_{\min} = \frac{1}{S}$, and so it follows that $-1/\left(\ln \left( 1-\frac{1}{S}\right) \right) \approx S$.  Then, Asynchronous VI's computational complexity $O\left(AS H_{\gamma, \epsilon} \left(\ln\left( \frac{S H_{\gamma, \epsilon}}{\delta} \right)/\ln\left( \frac{1}{1-p_{min}}\right) \right) \right) \approx O \left(S^2 A H_{\gamma, \epsilon} \ln\left( \frac{S H_{\gamma, \epsilon}}{\delta} \right)\right)$.  

\cite{chenwang2017} have established a lower bound on the computational complexity of planning in finite discounted MDP to be $\Omega(S^2 A)$.  The importance of their result shows that no algorithm can escape this $S^2 A$ computational complexity.  Both DAVI and Asynchronous VI computational complexity matches that of the lower-bound $\Omega(S^2 A)$ up to log terms in $S^2 A$ but have additional dependence on $H_{\gamma, \epsilon}$ and $\ln(1/\delta)$. 

VI, Asynchronous VI, and DAVI all include a horizon term. The horizon term $H_{\gamma, \epsilon}$ improves upon the horizon term of $H_{\gamma, \frac{\epsilon(1-\gamma)}{2 \gamma}} = \ln\left( \frac{2 \gamma \norm{v^* - v_0}}{\epsilon(1-\gamma)}\right) / (1- \gamma)$ that appears in the VI bound of \cite{vi_bound} when $\gamma > 0.5$.  As VI does not require sampling, it has no failure probability.  Thus, DAVI and Asynchronous VI both have an additional $\ln\left (\frac{1}{\delta} \right)$. We leave open the question of whether the additional log term $\ln(SH_{\gamma, \epsilon})$ in DAVI and Asynchronous VI is necessary.  

The computational complexity of DAVI nearly matches that of VI, but DAVI does not need to sweep through the action space in every state update.  Similar to Asynchronous VI, DAVI also does not need to wait for all states to complete their update in batches, as is the case of VI, making DAVI a more practical algorithm.

\section{Experiments} \label{sec:exp}
DAVI relies on sampling to reduce the computation of each update, and the performance of DAVI can be affected by the sparsity of rewards.  If a problem is like a needle in a haystack, where only one specific sequence of actions leads to a reward, then we do not anticipate uniform sampling to be beneficial in terms of total computation.  An algorithm would still have to consider most states and actions to make progress in this case.  On the other hand, we hypothesize that DAVI would converge faster than Asynchronous VI in domains with multiple optimal or near-optimal policies.  To isolate the effect of reward sparsity from the MDP structure, we first test our hypothesis on several single-state MDP domains.  However, solving a multi-state MDP is generally more challenging than solving a single-state MDP.  In our second experiment, we examine the performance of DAVI on two sets of MDPs: an MDP with a tree structure and a random MPD.  

The algorithms that will be compared in the experiments are VI, Asynchronous VI, and DAVI.  We implement Asynchronous VI and DAVI using uniform sampling to obtain the states.  DAVI samples a new set of actions via uniform sampling without replacement in each iteration.  

\subsection{Single-state experiment}
This experiment consists of a single-state MDP with $10000$ actions, all terminate immediately.  We experiment on two domains: needle-in-the-haystack and multi-reward.  Needle-in-the-haystack has one random action selected to have a reward of 1, with all other rewards set to 0.  Multi-reward has 10 random actions with a reward of 1.  The problems in this single-state experiment amount to brute-force search for the actions with the largest reward.

\subsection{Multi-state experiment}
This experiment consists of two sets of MDPs.  The first set consists of a tree with a depth of 2.  Each state has 50 actions, where each action leads to 2 other distinct next states.  All actions terminate at the leaf states.  Rewards are 0 everywhere except at a random leaf state-action pair, where reward is set to 1.  With this construct, there are around 10000 states.  The second set consists of a random MDP with 100 states, where each state has 1000 actions.  Each action leads to 10 next states randomly selected from the 100 states with equal probability.  All transitions have a 0.1 probability of terminating.  A single state-action pair is randomly chosen to have a reward of 1.  The $\gamma$ in all of the MDPs are 1.

\subsection{Discussion}
\Cref{fig:single} and \Cref{fig:multi} show the performance of the algorithms.  All graphs included error bars showing the standard error of the mean.  Notice that all graphs started at 0 and eventually reached an asymptote unique to each problem setting.  All graphs smoothly increased towards the asymptote except for Asynchronous VI in \Cref{fig:single} and VI in \Cref{fig:multi}, whose performances were step-functions \footnote{Asychronous VI in the single-state experiment is equivalent to VI since there is only one state.}.  The y-axis of each graph showed a state value averaged over 200 runs.  The x-axes showed run-times, which have been adjusted for computations.  

In \Cref{fig:single}(a,b), DAVI with $m = 1$ was significantly different from that of DAVI with $m=10,100, 1000$, and DAVI with $m=10,100,1000$ converged at a similar rate. While in \Cref{fig:multi}(a,b), all algorithms were significantly different.  DAVI $m = 10$ in random MDP \Cref{fig:multi}(b) converged faster than any other algorithms.  These results suggest that an ideal $m$ exists for each domain.
\begin{figure}[!ht]
    \centering
    \includegraphics[keepaspectratio=true,scale=0.43]{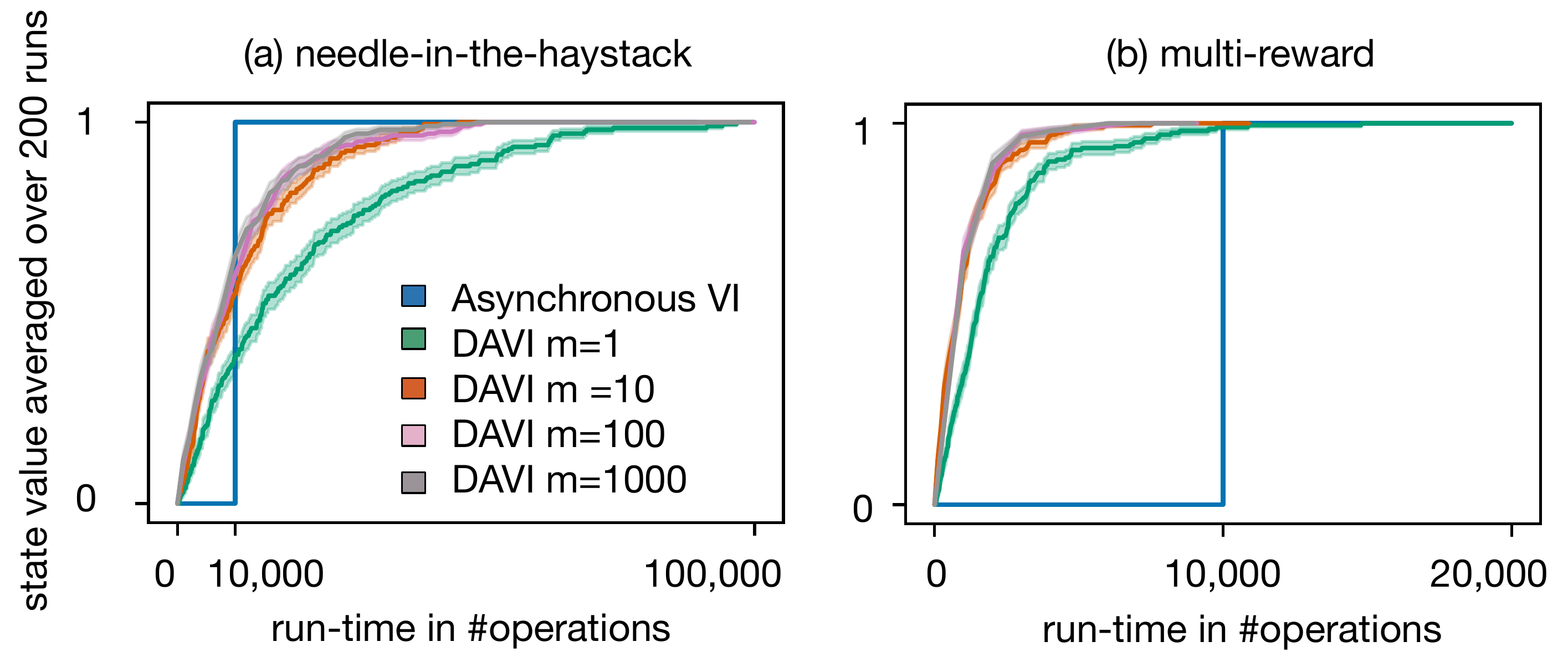}
    \caption{Single-state experiment with $10,000$ actions: (a) has only one action with a reward of 1 (b) has 10 actions with a reward of 1.  The Asynchronous VI in this experiment is equivalent to VI since there is only one state. We run each instance 200 times with a new MDP generated each time. In each run, each algorithm is initialized to 0.}
    \label{fig:single}
\end{figure}

In the needle-in-the-haystack setting \Cref{fig:single}(a), all four DAVI algorithms made some progress by the 10000 computation mark while Asynchronous VI stayed flat. Once Asynchronous VI finished computing the look-ahead value for all of the actions, it reached the asymptote immediately.  On the other hand, DAVI might have been lucky in some of the runs, found the optimal action amongst the subset early on, logged it as a best-so-far action, and had a state value sustained at 1 thereafter.  However, there could also be runs where sampling had the opposite effect.  

Changing the reward structure by introducing a few redundant optimal actions into the action space increased the probability that an optimal action was included in the subset.  In the multi-reward setting \cref{fig:single}(b), DAVI with all settings of $m$ have essentially reached the asymptote by the 10000 mark.  DAVI with all four action subset sizes reached the asymptote faster than Asynchronous VI.  As expected, DAVI converged faster than Asynchronous VI in the case of multiple rewarding actions. 

In the \Cref{fig:multi}(a), we saw a similar performance to that of the needle-in-the-haystack in the single-state experiment \Cref{fig:single}(a).  When we changed the MDP structure to allow for multiple possible paths that led to the special state with the hidden reward, as evident in \Cref{fig:multi}(b), DAVI with all settings of $m$ all reached the asymptote faster than Asynchronous VI and VI.  As expected, DAVI converged faster than Asynchronous VI in the case of multiple near-optimal policies. 
\begin{figure}[!ht]
    \centering
    \includegraphics[keepaspectratio=true,scale=0.4]{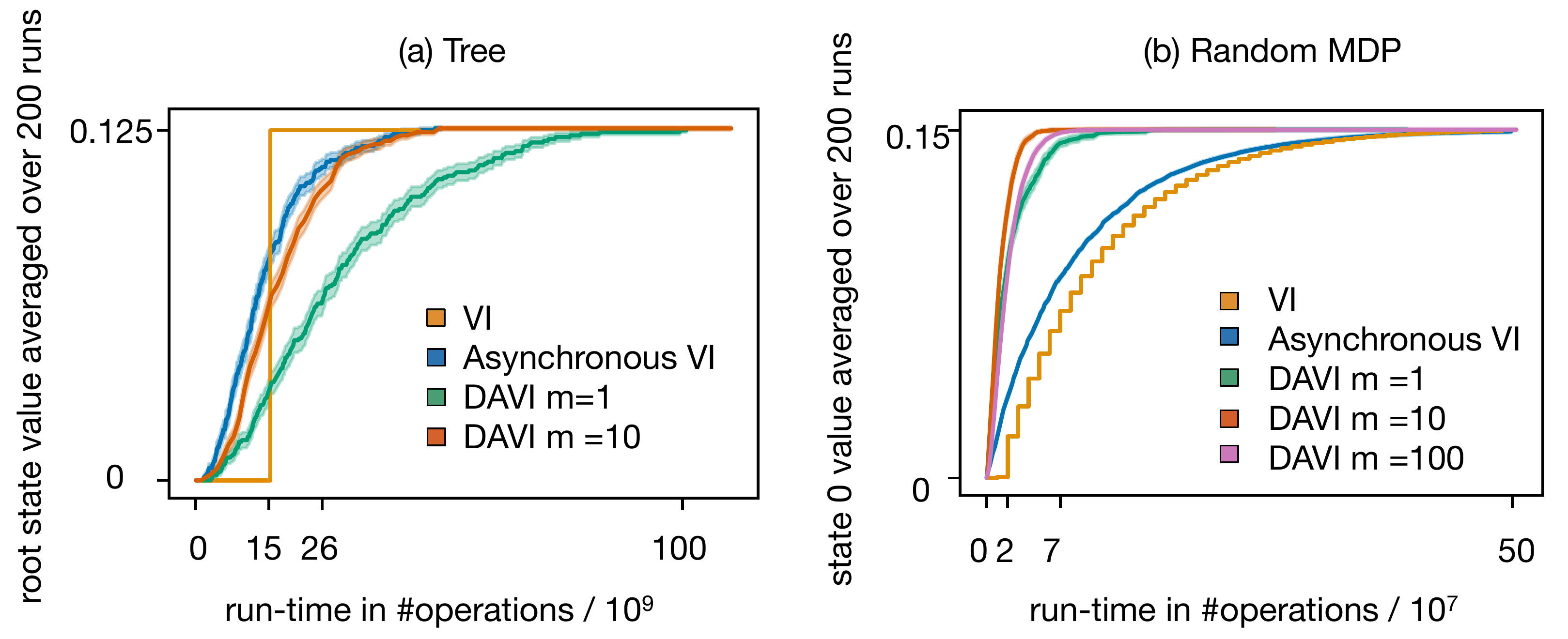}
    \caption{Multi-state-action needle-in-the-haystack experiment: (a) MDP with a tree structure (b) random MDP.  We run each instance 200 times with a new MDP generated each time. In each run, all algorithms are initialized to 0. }
    \label{fig:multi}
\end{figure}

We note that perhaps many other real-world problems resemble a setting like random-MDP more than needle-in-the-haystack, and hence the result on random-MDP may be more important.  See \cref{appendix:exp} for additions experiments with rewards drawn from Normal and Pareto distributions.   

\section{Conclusion}
The advantage of running asynchronous algorithms on domains with large state and action space have been made apparent in our studies.  Asynchronous VI helps to address the large state space problem by making in-place updates, but it is still intractable in domains with large action space.  DAVI is asynchronous in the state updates as in Asynchronous VI, but also asynchronous in the maximization over the actions.  We show DAVI converges to the optimal value function with probability 1 and at a near-geometric rate with a probability of at least $1 - \delta$.  Asynchronous VI also achieves a computational complexity closely matching that of VI.  We give empirical evidence for DAVI's computational efficiency in several experiments with multiple reward settings. 

Note that DAVI does not address the summation over the states: $\sum_{s' \in \cS} p(s' | s,a) v(s')$ in the computation of the look-ahead values.  If the state space is large, computing such a sum can also be prohibitively expensive.  Prior works by \cite{harm} use ``small back-ups'' to address this problem.  Instead of a summation of all successor states, they update each state's value with respect to one successor state in each update.  Another possibility is to sample a subset of successor states to compute the look-ahead values at the cost of additional failure probability.  Combining these techniques with DAVI is a potential direction for future work.

\begin{ack}
We thank Tadashi Kozuno, Csaba Szepesvári, and Roshan Shariff for their valuable comments and feedback. The authors gratefully acknowledge funding from DeepMind, Amii, NSERC, and CIFAR.
\end{ack}

\bibliographystyle{apalike}
\bibliography{ref}

\section*{Checklist}

The checklist follows the references.  Please
read the checklist guidelines carefully for information on how to answer these
questions.  For each question, change the default \answerTODO{} to \answerYes{},
\answerNo{}, or \answerNA{}.  You are strongly encouraged to include a {\bf
justification to your answer}, either by referencing the appropriate section of
your paper or providing a brief inline description.  For example:
\begin{itemize}
  \item Did you include the license to the code and datasets? \answerYes{}
  \item Did you include the license to the code and datasets? \answerNo{The code and the data are proprietary.}
  \item Did you include the license to the code and datasets? \answerNA{}
\end{itemize}
Please do not modify the questions and only use the provided macros for your
answers.  Note that the Checklist section does not count towards the page
limit.  In your paper, please delete this instructions block and only keep the
Checklist section heading above along with the questions/answers below.

\begin{enumerate}

\item For all authors...
\begin{enumerate}
  \item Do the main claims made in the abstract and introduction accurately reflect the paper's contributions and scope?
    \answerYes{}
  \item Did you describe the limitations of your work?
    \answerYes{}
  \item Did you discuss any potential negative societal impacts of your work?
    \answerNA{}
  \item Have you read the ethics review guidelines and ensured that your paper conforms to them?
    \answerYes{}
\end{enumerate}

\item If you are including theoretical results...
\begin{enumerate}
  \item Did you state the full set of assumptions of all theoretical results?
    \answerYes{}
        \item Did you include complete proofs of all theoretical results?
    \answerYes{}
\end{enumerate}

\item If you ran experiments...
\begin{enumerate}
  \item Did you include the code, data, and instructions needed to reproduce the main experimental results (either in the supplemental material or as a URL)?
    \answerNo{}
  \item Did you specify all the training details (e.g., data splits, hyperparameters, how they were chosen)?
    \answerNA{}
        \item Did you report error bars (e.g., with respect to the random seed after running experiments multiple times)?
    \answerYes{}
        \item Did you include the total amount of compute and the type of resources used (e.g., type of GPUs, internal cluster, or cloud provider)?
    \answerNo{}
\end{enumerate}

\item If you are using existing assets (e.g., code, data, models) or curating/releasing new assets...
\begin{enumerate}
  \item If your work uses existing assets, did you cite the creators?
    \answerNA{}
  \item Did you mention the license of the assets?
    \answerNA{}
  \item Did you include any new assets either in the supplemental material or as a URL?
    \answerNA{}
  \item Did you discuss whether and how consent was obtained from people whose data you're using/curating?
    \answerNA{}
  \item Did you discuss whether the data you are using/curating contains personally identifiable information or offensive content?
    \answerNA{}
\end{enumerate}

\item If you used crowdsourcing or conducted research with human subjects...
\begin{enumerate}
  \item Did you include the full text of instructions given to participants and screenshots, if applicable?
    \answerNA{}
  \item Did you describe any potential participant risks, with links to Institutional Review Board (IRB) approvals, if applicable?
    \answerNA{}
  \item Did you include the estimated hourly wage paid to participants and the total amount spent on participant compensation?
    \answerNA{}
\end{enumerate}

\end{enumerate}


\appendix
\section*{Supplementary material}
\section{Auxilary proofs for DAVI's theoretical results}\label{appendix:davi}
This section shows the proof of the supporting lemmas required in the proof of DAVI's convergence and convergence rate.    We also include here a more general proof of the convergence of DAVI and each of the corollaries.    The numbering of each lemma, corollary, and theorem corresponds to the main paper's numbering.

\begin{definition} \label{def:appendex_davi_op}
    Recall $T_{n}: \R^S \to \R^S$. For a given $\A_n \sim \q$, $\pi_n \in \Pi$, $s_n \in \cS$, and for all $s \in \cS$ and $v \in \R^S$, 
\begin{align}
T_{n}v(s) \doteq
\begin{cases}
    \max_{a \in \A_n \cup \pi_n(s)} L^v(s,a) & \text{if } s = s_n \\
    v(s) & \text{otherwise}. \label{eq:davi_operator}
\end{cases}
\end{align}
Define $T_{\pi, s_n}: \R^S \to \R^S$. For a given $\pi \in \Pi$, $s_n \in \cS$, and for all $s \in \cS$ and $v \in \R^S$,
\begin{align}
     T_{\pi, s_n}v(s) &\doteq \begin{cases}
        L^{v}(s, \pi(s)) & \text{if } s = s_n \\
    v(s) & \text{otherwise}.
    \end{cases} \label{eq:davi_t_pi}
\end{align}
Then, the value iterates of DAVI evolves according to $v_{n+1} = T_{n} v_n$ for all $n \in \mathbb{N}_0$.  Alternatively, $v_{n+1} = T_{\pi_{n+1}, s_n} v_n$ with $\pi_{n+1}(s)$ being the the action that satisfies $\max_{a \in \A_n \cup \pi_n(s)} L^{v_n}(s,a)$ for $s = s_n$ and $\pi_{n+1}(s) = \pi_n(s)$ for $s \neq s_n$. $(i.e., T_{\pi_{n+1}, s_n}v_n = T_n v_n)$. 
\end{definition}

\begin{definition}[Optimality capture region \citep{Williams93analysisof}] \label{def:gap_opt_cap}
Define 
 \begin{align}
    \Delta^{v}(s) = \min\left[ \left\{  \max_{a' \in \A} L^{v}(s, a') - L^{v}(s, a) \Big | a \in \A \right \} - \{ 0 \} \right] \label{eq:gap}
\end{align}    
as the difference between the look-ahead value with respect to $v$ of the greedy action and a second-best action for state $s$.  Let $\Delta^{v^*} \doteq \min_{s \in \cS} \Delta^{v^*}(s)$.  Then, the optimality capture region is defined to be
\begin{align}
    \left \{ v: \norm{v^* - v} < \frac{\Delta^{v^*}}{2\gamma}, v \in \R^S \right \}.
\end{align}
\end{definition}

\begin{lemma}\label{appendix:monotone_op}
    DAVI operators $T_n$ and $T_{\pi, s'}$ are monotone operators.  That is given $v, u \in \R^S$ if $v \leq u$, then $T_n v \leq T_n u$ and $T_{\pi, s'} v \leq T_{\pi, s'} u$. 
\end{lemma}
\begin{proof}
Given any $v, u \in \R^S$ s.t. $v \leq u$, then
\begin{align}
    T_n v(s) &= \begin{cases}
        \max_{a \in \A_n \cup \pi_n(s)} r(s,a) + \gamma \sum_{s'}p(s'|s,a)v(s') & \text{if } s = s_n \\
        v(s) & \text{otherwise}
    \end{cases}\\
    & \leq \begin{cases}
        \max_{a \in \A_n \cup \pi_n(s)} r(s,a) + \gamma \sum_{s'}p(s'|s,a)u(s') & \text{if } s = s_n \\
        u(s) & \text{otherwise}
    \end{cases} \\
    & = T_n u(s).
\end{align}

Given any $v,u \in \R^S$ s.t. $v \leq u$, then
\begin{align}
     T_{\pi, s_n}v(s) &\doteq \begin{cases}
        r(s, \pi(s)) + \gamma \sum_{s'} p(s'|s,\pi(s)) v(s') & \text{for } s = s_n \\
    v(s) & \text{otherwise}
    \end{cases} \\
    &\leq \begin{cases}
        r(s, \pi(s)) + \gamma \sum_{s'} p(s'|s,\pi(s)) u(s') & \text{for } s = s_n \\
     u(s) & \text{otherwise}
    \end{cases}\\
    &= T_{\pi, s_n}u(s).
\end{align}
\end{proof}

\begin{customlemma}{1}[Monotonicity] \label{appendix:monoton}
   The iterates of DAVI, $(v_n)_{n \geq 0}$ is a monotonically increasing sequence: $v_n \leq v_{n+1}$ for all $n \in \mathbb{N}_{0}$, if $r(s,a) \in [0,1]$ for any $s,a \in \cS \times \A$ and if DAVI is initialized according to (i),(ii), or (iii) of \cref{asumpt:init}.  
\end{customlemma}
\begin{proof}
    We show $(v_n)_{n \geq 0}$  is a monotonically increasing sequence by induction.  All inequalities between vectors henceforth are element-wise.  Let $(s_0, s_1,...,s_n, s_{n+1})$ be the sequence of states sampled for update from iteration $1$ to $n+1$. By straight-forward calculation, we show $v_1 \geq v_0$.  For all rewards in $[0,1]$ and for any $s \in \cS$,
\begin{align}
    \text{case } i: v_1(s) &= \max_{a \in \A_0 \cup \pi_0(s)} \left \{r(s, a) + \gamma \sum_{s'}p(s'|s,a) 0 \right\} \\ 
    & \geq r(s, \pi_0(s)) + \gamma \sum_{s'}p(s'|s,\pi_0(s)) 0 \\
    &= L^{v_0}(s, \pi_0(s)) \geq 0 = v_0(s)\\
    \text{case } ii: v_1(s) &= \max_{a \in \A_0 \cup \pi_0(s)} \left \{r(s,a) + \gamma  \sum_{s'}p(s'|s,a)(-c) \right\} \\
    & \geq r(s, \pi_0(s)) + \gamma  \sum_{s'}p(s'|s,\pi_0(s)))(-c) = L^{v_0}(s, \pi_0(s))\\
    &= -\gamma c + r(s, \pi_0(s)) \geq -c = v_0(s) \\
    \text{case } iii: v_1(s) &= \max_{a \in \A_0 \cup \pi_0(s)} \left \{r(s,a) + \gamma  \sum_{s'}p(s'|s,a)v_0(s') \right\} \\
    & \geq r(s, \pi_0(s)) + \gamma \sum_{s'}p(s'|s,\pi_0(s)) v_0(s') = L^{v_0}(s, \pi_0(s)) \\
    &\geq v_0(s) \quad \text{by assumption}.
\end{align}
Thus, $v_1(s_0) \geq v_1(s_0)$.  For all other states $s \neq s_0$, $v_0(s) = v_1(s)$.  Therefore, $v_1 \geq v_0$.  Now, assume $v_n \geq \dots \geq v_{0}$ with $n \geq 1$, then for any $s \in \cS$,
\begin{align}
    v_{n+1}(s) &= \Topt{n}v_n(s)  \\
    & = \begin{cases}
        \max_{a \in \A_n \cup \pi_n(s)} L^{v_n}(s,a) & \text{if } s = s_{n} \\
        v_n(s) & \text{otherwise } 
    \end{cases} \\
    &\geq \begin{cases}
        L^{v_n}(s,\pi_n(s)) & \text{if } s = s_{n} \\
        v_n(s) & \text{otherwise } 
    \end{cases} \\
    & \geq \begin{cases}
        L^{v_{n-1}}(s,\pi_n(s)) & \text{if } s = s_{n}  \quad \text{ by assumption $v_n \geq v_{n-1}$}\\
        v_{n-1}(s) & \text{otherwise }. \label{eq:18}
    \end{cases}
\end{align}
If $s_n = s_{n-1}$, then \cref{eq:18} is $\Topt{\pi_n, s_{n-1}} v_{n-1}$. By \cref{def:appendex_davi_op}, $\Topt{\pi_n, s_{n-1}} v_{n-1}= v_n$. Hence, $v_{n+1} \geq v_{n}$.  However, if $s_n \neq s_{n-1}$, we have to do more work.  There are two possible cases. The first case is that $s_n$ has been sampled for update before.  That is, let $1 < j \leq n$ s.t. $s_{n-j}$ is the last time that $s_n$ is sampled for update.  Then $s_n = s_{n-j}$, and $v_n(s_n) = v_{n-j+1}(s_n)$ and $\pi_{n}(s_n) = \pi_{n-j+1}(s_n)$.  By assumption, $v_{n} \geq ... \geq v_{n-j} \geq ... \geq v_0$, then 
\begin{align}
    v_{n+1}(s_n) &= 
        \max_{a \in \A_n \cup \pi_n(s_n) } L^{v_n}(s_n, a) \geq L^{v_n}(s_n, \pi_n(s_n))  \\
    &\geq L^{v_{n-j}}(s_n, \pi_n(s_n)) \quad \text{by assumption } v_n \geq v_{n-j}\\
    &= L^{v_{n-j}}(s_n, \pi_{n-j+1}(s_n)) \\
    &= T_{\pi_{n-j+1}, s_{n-j} }v_{n-j}(s_n) \quad \text{by \cref{eq:davi_t_pi}}\\
    &= v_{n-j+1}(s_n)  \quad \text{by \cref{def:appendex_davi_op}}\\
    &= v_n(s_n).
\end{align}
We have just showed that $v_{n+1}(s_n) \geq v_{n}(s_n)$, and for all other state $s \neq s_n$, $v_{n+1}(s) = v_{n}(s)$. For the second case, $s_n$ has not been sampled for updated before $n$, then $v_n(s_n) = v_0(s_n)$ and $\pi_n(s_n) = \pi_{0}(s_n)$.  By assumption, $v_n \geq ... \geq v_0$, then
\begin{align}
    v_{n+1}(s_n) &= 
        \max_{a \in \A_n \cup \pi_n(s_n) }L^{v_n}(s_n, a) \geq L^{v_n}(s_n, \pi_n(s_n)) \\
        &\geq L^{v_0}(s_n, \pi_n(s_n)) \\
        &= L^{v_0}(s_n, \pi_0(s_n)) \geq v_0(s_n) \quad \text{shown in base case}\\
        &= v_n(s_n).
\end{align}
For all other state $s \neq s_n$, $v_{n+1}(s) = v_{n}(s)$.  Altogether, $v_{n+1} \geq v_n$ for all $n \in \mathbb{N}_0$.  
\end{proof}

\begin{customcorollary}{1}[Computational complexity of obtaining an $\epsilon$-optimal policy] \label{proof:complexity}
 Fix an $\epsilon \in (0, \norm{v^* - v_0})$, and assume DAVI initializes according to (i), (ii), or (iii) of \cref{asumpt:init}.  Define 
 \begin{align}
     H_{\gamma, \epsilon} \doteq \ln \left(\frac{\norm{v^* - v_0}}{\epsilon}\right)/{1-\gamma}
 \end{align}
 as a horizon term.  Then, DAVI runs for at least
 \begin{align}
     \tau = H_{\gamma, \epsilon} \left(\ln\left(\frac{S H_{\gamma, \epsilon}}{\delta} \right)/\ln\left( \frac{1}{1-q_{min}}\right)\right)
 \end{align}
 iterations, returns an $\epsilon$-optimal policy $\pi_n: v_{\pi_n} \geq v^* - \epsilon \1$ with probability at least $1-\delta$ using $O\left(mS \tau \right)$ elementary arithmetic and logical operations.  Note that $\norm{v^* - v_0}$ is unknown but it can be upper bounded by $\frac{1}{1-\gamma} + \norm{v_0}$ given rewards are in $[0,1]$.
\end{customcorollary}
 \begin{proof}
Recall from \cref{lemma:monotone}, DAVI's value iterates, $v_n \to v^*$ monotonically from below (i.e., $v_n \geq v_{n-1} \geq \dots \geq v_0$ ).  Using this result, one can show $L^{v_n}(s, \pi_n(s)) \geq v_n(s)$ for all $s \in \cS$ and $n \in \mathbb{N}_0$ following an induction process.  We have already shown in the proof \cref{lemma:monotone} that $L^{v_0}(s, \pi_0(s)) \geq v_0(s)$ for any $s \in \cS$ in the base case.  Assume that $L^{v_n}(s, \pi_n(s)) \geq v_n(s)$ for any $s \in \cS$, we will show that $L^{v_{n+1}}(s, \pi_{n+1}(s)) \geq v_{n+1}(s)$.  For any $n \in \mathbb{N}_0$ and $s_n \in \cS$, let $\pi_{n+1}(s_n) = \argmax_{a \in \A_n \cup \pi_n(s_n)} L^{v_n}(s_n,a)$ with $\pi_{n+1}(\bar s) = \pi_{n}(\bar s)$ for all other $\bar s \neq s_n$.  

For the case when $s = s_n$,
\begin{align}
    v_{n+1}(s) = T_{\pi_{n+1}, s_n} v_n(s) &= L^{v_n}(s, \pi_{n+1}(s)) \\
    & \leq L^{v_{n+1}}(s, \pi_{n+1}(s)) \quad \text{by } v_n \leq v_{n+1}.
\end{align}
For the case when $s \neq s_n$, then $v_{n+1}(s) = v_n(s)$ and $\pi_{n+1}(s) = \pi_{n}(s)$, and thus
\begin{align}
    v_{n+1}(s) = T_{\pi_{n+1}, s_n} v_n(s) &= v_n(s) \leq L^{v_n}(s, \pi_n(s)) \quad \text{by assumption}\\
    &= L^{v_{n+1}}(s, \pi_{n+1}(s)).
\end{align}
Altogether, we get $L^{v_{n+1}}(s, \pi_{n+1}(s)) \geq v_{n+1}(s)$ for any $s \in \cS$, which concludes the induction.

Now, we show that $v_{\pi_n} \geq v_n$ for any $n \in \mathbb{N}_0$ using the result $L^{v_n}(s, \pi_n(s)) \geq v_n(s)$ for any $s \in \cS$ and $n \in \mathbb{N}_0$.  Fix $n$ and if we are to apply the policy evaluation operator $T_{\pi_n}$ that satisfy \cref{lemma:fixpoint}(1) to every state $s \in \cS$, then we obtain
\begin{align}
    T_{\pi_n}v_n(s) = L^{v_n}(s, \pi_n(s)) \geq v_n(s).
\end{align}
Therefore, $T_{\pi_n}v_n \geq v_n$.  By applying the $T_{\pi_n}$ operator to $T_{\pi_n} v_n \geq v_n$ repeatedly and by using the monotonicity of $T_{\pi_n}$, we have for any $k \geq 0$, 
\begin{align}
    T^{k}_{\pi_n}v_n &\geq T^{k-1}_{\pi_n}v_n \geq \cdots \geq v_n.
\end{align}
By taking limits of both sides of $T_{\pi_n}^k v_n \geq v_n$ as $k \to \infty$, we get $v_{\pi_n} \geq v_n$.  Therefore, 
\begin{align}
    0 \leq v^* - v_{\pi_n} \leq v^* - v_n \implies
    \norm{v^* - v_{\pi_n}} \leq \norm{v^* - v_n}.\label{eq:48}
\end{align}

Next, recall from the proof of \cref{thm:conv_rate} that for a given $l \in \mathbb{N}$, and with probability $1 - \delta$, $v_n$ of DAVI would have $\gamma$-contracted at least $l$ times: $\norm{v^* - v_n} \leq \gamma^l \norm{v^* - v_0}$, with 
$n \geq l \left \lceil \ln \left( \frac{Sl}{\delta} \right) \Big / \ln \left( \frac{1}{1 - q_{min}}\right) \right \rceil$.  
Following from \cref{eq:48}, with probability $1 - \delta$, 
\begin{align}
    \norm{v^* - v_{\pi_n}} \leq \norm{v^* - v_n} \leq \gamma^l \norm{v^* - v_0}. 
\end{align}
By setting $\gamma^l \norm{v^* - v_0} = \epsilon$ and solve for $l$, we get:
\begin{align}
    l &= \ln{\frac{\norm{v^*-v_0}}{\epsilon}}/\ln\left(\frac{1}{\gamma} \right). 
\end{align}
We observe that $
    \ln{\left(\frac{\norm{v^*-v_0}}{\epsilon}\right)}/\ln\left(\frac{1}{\gamma} \right) \leq \ln{\left(\frac{\norm{v^*-v_0}}{\epsilon}\right)}/(1-\gamma) \doteq H_{\gamma, \epsilon}.$
To compute $v_n$, DAVI takes $O(mS)$ elementary arithmetic operations.  With probability $1-\delta$, DAVI obtains an $\epsilon$-optimal policy with
\begin{align}
    O(mSn) &= 
   O\left(mS H_{\gamma, \epsilon}\ln\left( \frac{S H_{\gamma, \epsilon}}{\delta} \right)/\ln\left( \frac{1}{1-q_{min}}\right)\right) \label{eq:o_complex}
\end{align}
arithmetic and logical operations.
\end{proof}

\begin{customcorollary}{2}[Computational complexity of obtaining an optimal policy] \label{appendix:opt_policy_complexity}
 Assume DAVI initializes according to (i), (ii), or (iii) of \cref{asumpt:init}.  Define the horizon term
\begin{align}
     H_{\gamma, \Delta^{v^*}} \doteq \ln\left(\frac{\norm{v^*-v_0}}{\Delta^{v^*}} \right)/(1-\gamma),
 \end{align}
 where $\Delta^{v^*}$ is the optimality capture region defined in \cref{def:gap_opt_cap}. Then, DAVI returns an optimal policy $\pi^* \in \Pi^*$ with probability $1-\delta$, requiring
 \begin{align}
     O \left( mS H_{\gamma, \Delta^{v^*}} \ln\left( \frac{SH_{\gamma, \Delta^{v^*}}}{\delta} \right)/\ln\left( \frac{1}{1-q_{min}}\right) \right) 
 \end{align}
 elementary arithmetic operations.  Note that $\norm{v^* - v_0}$ is unknown but it can be upper bounded by $\frac{1}{1-\gamma} + \norm{v_0}$ given rewards are in $[0,1]$.
 \end{customcorollary}
\begin{proof}
We first show that any $\pi_n$ such that $v^{\pi_n}>v^*- \Delta^{v^*} \1$ is an optimal policy.  We prove this by contradiction.  Assume $\pi_n$ is not optimal but satisfies $v_{\pi_n} > v^* - \Delta^{v^*} \1$, then for any $s \in \cS$
\begin{align}
    L^{v^*}(s, \pi_n(s)) &< L^{v^*}(s, \pi^*(s)) \\
    &\implies L^{v^*}(s, \pi^*(s)) - L^{v^*}(s, \pi_n(s)) > 0 \\
    &\implies L^{v^*}(s, \pi^*(s)) - L^{v^*}(s, \pi_n(s)) \geq \Delta^{v^*}\quad \text{by \cref{def:gap_opt_cap}} \\
    &\implies L^{v^*}(s, \pi^*(s)) - L^{v_{\pi_n}}(s, \pi_n(s)) \geq \Delta^{v^*}\\
    &\implies v^*(s) - v_{\pi_n}(s) \geq \Delta^{v^*} \\
    &\implies v_{\pi_n}(s) \leq v^*(s) - \Delta^{v^*}.
\end{align}
This contradicts the assumption and $\pi_n$ must be optimal. It is straight-forward to show that the result of \cref{cor:complexity} still holds if we require $\pi_n: v_{\pi_n} > v^* - \epsilon \1$ instead of $\pi_n: v_{\pi_n} \geq v^* - \epsilon \1$. We can then apply this result to show that DAVI returns policy $\pi_n$ such that $\pi_n : v_{\pi_n} > v^* - \Delta^{v^*} \1$, and thus an optimal policy, with probability $1-\delta$ within 
\begin{align}
   O\left(mS H_{\gamma, \Delta^{v^*}} \left(\ln\left( \frac{S H_{\gamma, \Delta^{v^*}}}{\delta} \right)/\ln\left( \frac{1}{1-q_{min}}\right)\right)\right)
\end{align}
arithmetic and logical operations. 
\end{proof}

Now we show an alternative proof to the convergence of DAVI with any initialization.  Before we prove the main result, we define the following supporting lemmas.
\begin{lemma}[\cite{Williams93analysisof}] \label{lemma:close_enough}
  Let $v, u \in \R^S, s \in \cS$.  
  
  Let $\pi(s) = \argmax_{a \in \A} L^v(s,a)$ and an $a \in \A$ satisfies $L^u(s,a) \geq L^u(s, \pi(s))$.  Then 
  \begin{align}
      \norm{v - u} < \frac{\Delta^{v}}{2\gamma} \label{eq:capture_region}
  \end{align}
  implies that $L^v(s,\pi(s)) = L^{v}(s,a)$.  
\end{lemma}

\begin{lemma} \label{lemma:pi_optimal}
  Given $v \in \R^S$ which satisfies $\norm{v^* - v} < \frac{\Delta^{v^*}}{2\gamma}$ (i.e., $v$ is inside the optimality capture region), if an action $a$ satisfies $L^v(s,a) = \max_{a' \in \A} L^{v}(s,a')$, then $a$ is an optimal action at $s$. 
\end{lemma}
\begin{proof}
For any $s \in \cS$, let the optimal policy at $s$ be $\pi^*(s) = \argmax_{a \in \A} L^{v*}(s,a)$ and $\pi(s) = \argmax_{a \in \A} L^v(s,a)$, then
\begin{align}
    L^v(s,\pi(s)) \geq L^v(s, \pi^*(s)).
\end{align}
Since $\norm{v^* - v} < \frac{\Delta^{v^*}}{2\gamma}$ and by \cref{lemma:close_enough}, $L^{v^*}(s, \pi^*(s)) = L^{v^*}(s, \pi(s))$.  
\end{proof}

\begin{lemma}[Stochastically always \citep{Williams93analysisof}] \label{lemma:stochastic_always}
  Let $X$ be a set of finite operators on $\A^S \times \R^S$.    We say a stochastic process is stochastic always if every operator in $X$ has a non-zero probability of being drawn.    Let $\Sigma$ be an infinite sequence operator from $X$ generated by a stochastic always stochastic process.    Let $\Sigma'$ be a given finite sequence of operators from $X$, then
   \begin{enumerate}
       \item $\Sigma'$ appears as a contiguous subsequence of $\Sigma$ with probability 1, and
       \item $\Sigma'$ appears infinitely often as a contiguous subsequence of $\Sigma$ with probability 1.
   \end{enumerate}
\end{lemma}

\begin{theorem}[Convergence of DAVI with any initialisation]
   Let $\tilde \A$ be some arbitrary action subset of $\A$, and let $X = \{I_{\tilde \A, s},  T_{s} | s \in \cS \}$ be a set of DAVI operators that operate on $\A^S \times \R^S$ that is the joint space of policy and value function, where
   \begin{align}
       \pi_{n+1}(s)=I_{\tilde \A, s_n} \pi_n (s) = \begin{cases}
           \argmax_{a \in \tilde \A \cup \pi_n(s)} L^{v_n}(s,a) & \text{if } s=s_n \\
           \pi_n(s) & \text{otherwise},
       \end{cases} \label{eq:policy_improv}
   \end{align}
   and
   \begin{align}
       v_{n+1}(s) = T_{s_n} v_n(s) = \begin{cases}
           L^{v_n}(s, \pi_{n+1}(s)) & \text{if } s = s_n \\
           v_n(s) & \text{otherwise}. \label{eq:policy_eval}
       \end{cases}
   \end{align}
    Recall $\Pi$ is a set of deterministic policies defined in \cref{sec:background} and $\pi^* \in \Pi$.  Without loss of generality, we write $\cS = 1,...,S$. If DAVI performs the following sequence of operations in some fixed order,    
  \begin{align}
      I_{\tilde \A_1, 1}  T_{1} I_{\tilde \A_2, 2}  T_{2} ... I_{\tilde \A_S, S} T_{S}, \label{eq:good_op}
  \end{align}
  where $\tilde \A_i$ contains the optimal action $\pi^*(i)$ for state $i$, then $v_n$ would have $\gamma$-contracted at least once by the same argument as in the proof of \cref{thm:conv_rate}.  Let $\Sigma'$ be a concatenation of $l$ copies of a sequence \cref{eq:good_op}.  Then, after having performed all the operations in $\Sigma'$, $v_n$ would have $\gamma$-contracted $l$ times.  If $l$ satisfies:
  \begin{align}
      \gamma^l \norm{v^* - v_n} < \frac{\Delta^{v^*}}{2\gamma},
  \end{align}
  then $v_n$ is inside the optimality capture region defined in  \cref{def:gap_opt_cap}.  Once inside the optimality capture region, by \cref{lemma:pi_optimal}, all policies $\pi_n$ are optimal thereafer.  We know from \cref{lemma:fixpoint} (1), $\lim_{n \to \infty} T_{\pi^*} v = v^*$ and by \cref{lemma:bounded} (Boundedness), all $v_n$'s are bounded. Then, the convergence of DAVI with any initialization is ensured as long as all of the states are sampled for update infinitely often. 
  
  The only question is whether if $\Sigma'$ would ever exist in an infinite sequence $\Sigma$ that is generated by running DAVI forever.  To show that such event happens with probability 1, we apply \cref{lemma:stochastic_always}.  To apply \cref{lemma:stochastic_always} (Stochastically always), $X$ must be finite, which indeed it is since the state and action space are finite.  Ensuring that the $\q(s,a)> 0$ guarantees every operator in $X$ is drawn with a non-zero probability.  Therefore, the stochastic process generated by running DAVI would satisfy all the properties of \cref{lemma:stochastic_always}. By \cref{lemma:stochastic_always}, running DAVI forever will generate any contiguous subsequence $\Sigma'$ infinitely often with probability 1. 
  
\end{theorem}

\section{Theoretical analysis of Asynchronous VI} \label{appendix:avi}
\cite{neuro} and \cite{Williams93analysisof} have shown Asynchronous VI converges.  We can view Asynchronous VI as a special case of DAVI if the subset of actions sampled in each iteration is the entire action space.  That is for any $s \in \cS$, $v \in \R^S$ and $\pi \in \Pi$, $\max_{a \in \A \cup \pi(s)} L^{v}(s,a) = \max_{a \in \A} L^{v}(s,a)$.  We can follow similar reasoning to the proof of the convergence rate of DAVI (\cref{thm:conv_rate} )and show the convergence rate of Asynchronous VI with the $T$ operator defined in \cref{def:avi_op}. However, the sequence of increasing indices $(n_k)_{k=0}^{\infty}$, where $n_0 = 0$ in \cref{thm:conv_rate} takes on a slightly different meaning.  In particular, between the $n_k$-th and $n_{k+1}$-th iteration, all $s \in \cS$ have been updated at least once.  Finally, the computational complexity bound of Asynchronous VI is similar to the computational complexity bound of DAVI with $p_{min} = \min_{s} p(s)$ instead of $q_{min}$.  The computational complexity result is proven similarly to the proof of \cref{cor:complexity} found in \cref{appendix:davi}. 

\begin{definition}[Asynchronous VI operator] \label{def:avi_op}
    Recall $T_{s_n}: \R^S \to \R^S$. For a given $s_n \in \cS$, and for all $s \in \cS$ and $v \in \R^S$, 
\begin{align}
T_{s_n}v(s) \doteq
\begin{cases}
    \max_{a \in \A} L^v(s,a) & \text{if } s = s_n \\
    v(s) & \text{otherwise}. \label{eq:avi_operator}
\end{cases}
\end{align}
Then the iterates of Asynchronous VI evolves according to $v_{n+1} = T_{s_n} v_{n}$ for all $n \in \mathbb{N}_0$.  
\end{definition}

\begin{lemma}[Asynchronous VI Monotonicity] \label{lemma:avi_monotone}
  The iterates of Asynchronous VI, $(v_n)_{n\geq 0}$ is a monotonically increasing sequence: $v_n \leq v_{n+1}$ for all $n \in \mathbb{N}_{0}$, if $r(s,a) \in [0,1]$ for any $s,a \in \cS \times \A$ and if Asynchronous VI is initialized according to (i) or (ii) of \cref{asumpt:init}.
\end{lemma}
\begin{proof}
    We show $(v_n)_{n \geq 0}$ is a monotonically increasing sequence by induction.  All inequalities between vectors henceforth are element-wise.  Let $(s_0, s_1,...,s_n, s_{n+1})$ be the sequence of states sampled for update from iteration $1$ to $n+1$. By straight-forward calculation, we show $v_1 \geq v_0$.  For all rewards in $[0,1]$ and any $s \in \cS$,
\begin{align}
    \text{case } i: v_1(s) &= \max_{a \in \A} \left \{r(s, a) + \gamma \sum_{s'}p(s'|s,a) 0 \right\} \geq v_0(s)\\
    \text{case } ii: v_1(s) &= \max_{a \in \A} \left \{r(s,a) + \gamma  \sum_{s'}p(s'|s,a)(-c) \right\} \\
    &= \max_{a \in \A} \left\{ -\gamma c + r(s, a) \right\} \geq v_0(s).
\end{align}
Thus, $v_1(s_0) \geq v_0(s_0)$.  For all other states $s \neq s_0, v_0(s) = v_1(s)$.  Therefore, $v_1 \geq v_0$.  Now, assume $v_n \geq \dots \geq v_{0}$ with $n \geq 1$, then for any $s \in \cS$,
\begin{align}
    v_{n+1}(s) &= T_{s_n} v_n(s)  \\
    & = \begin{cases}
        \max_{a \in \A} L^{v_n}(s,a) & \text{if } s = s_{n} \\
        v_n(s) & \text{otherwise } 
    \end{cases} \\
    &\geq \begin{cases}
       \max_{a \in \A} L^{v_{n-1}}(s,a) & \text{if } s = s_{n}  \quad \text{ by assumption $v_n \geq v_{n-1}$} \\
        v_{n-1}(s) & \text{otherwise }. \label{eq:74}
    \end{cases} 
\end{align}
If $s_n = s_{n-1}$, then \cref{eq:74} is $T_{s_{n-1}} v_{n-1}$. By \cref{def:avi_op}, $T_{s_{n-1}} v_{n-1}= v_n$. Hence, $ v_{n+1} \geq v_n$.  However, if $s_n \neq s_{n-1}$, we have to do more work.  There are two possible cases. The first case is that $s_n$ has been sampled before.  That is, let $1 < j \leq n$ s.t. $s_{n-j}$ is the last time that $s_n$ is sampled for update.  Then $s_n = s_{n-j}$, and $v_n(s_n) = v_{n-j+1}(s_n)$.  By assumption, $v_{n} \geq ... \geq v_{n-j} \geq ... \geq v_0$, then 
\begin{align}
    v_{n+1}(s_n) &= 
        \max_{a \in \A } L^{v_n}(s_n, a) \\
    &\geq \max_{a \in \A} L^{v_{n-j}}(s_{n-j}, a) \quad \text{by assumption } v_n \geq v_{n-j}\\
    &= T_{s_{n-j}} v_{n-j}(s_{n-j}) = v_{n-j+1}(s_{n-j}) = v_n(s_n).
\end{align}
We have just showed that $v_{n+1}(s_n) \geq v_{n}(s_n)$, and for all other state $s \neq s_n$, $v_{n+1}(s) = v_{n}(s)$. For the second case, $s_n$ has not been sampled before $n$, then $v_n(s_n) = v_0(s_n)$.  By assumption, $v_n \geq ... \geq v_0$, then
\begin{align}
    v_{n+1}(s_n) &= 
        \max_{a \in \A}L^{v_n}(s_n, a) \\
        &\geq \max_{a \in \A} L^{v_0}(s_n, a) \quad \text{ by assumption } v_n \geq v_0 \\
        &\geq v_0(s_n) \quad \text{ shown in base case}.
\end{align}
For all other state $s \neq s_n$, $v_{n+1}(s) = v_{n}(s)$.  Altogether, $v_{n+1} \geq v_n$ for all $n \in \mathbb{N}_0$.
\end{proof}

\begin{theorem}[Convergence rate of Asynchronous VI] \label{thm:avi_conv_rate}
   Assume $p(s) > 0$ and $r(s,a) \in [0,1]$ for any $s, a \in \cS \times \A$, and also assume Asynchronous VI initialises according to (i), (ii) of \cref{asumpt:init}.  With $\gamma \in [0,1)$ and probability $1-\delta$, the iterates of Asynchronous VI, $(v_n)_{n \geq 0}$ converges to $v^*$ at a near-geometric rate. In particular, with probability $1-\delta$, for a given $l \in \mathbb{N}$,
   \begin{align}
       \norm{v^* - v_n} \leq \gamma^l \norm{v^* - v_0},
   \end{align}
   for any n satisfying
   \begin{align}
   n \geq l \left \lceil  \ln \left( \frac{S l}{\delta} \right) \Big / \ln \left(\frac{1}{ 1 - p_{min}}\right) \right \rceil,
   \end{align}
   where $p_{min} = \min_{s} p(s)$.
\end{theorem}
\begin{proof}
Recall from \cref{lemma:avi_monotone}, we have shown the iterates of Asynchronous VI, $v_n \to v^*$ monotonically from below.  We define
$(n_{k})_{k=0}^{\infty}$ to be a sequence of increasing indices, where $n_0 = 0$, such that  between the $n_k$-th and $n_{k+1}$-th iteration, all state $s \in \cS$ have been updated at least once.  At the $n_{k+1}$-th iteration, $v_{n_{k+1}} \geq T_{\pi^*}v_{n_k}$.  This implies that at the $n_{k+1}$-th iteration, Asynchronous VI would have $\gamma$-contracted at least once:
   \begin{align}
0 \leq v^* - v_{n_{k+1}}  &\leq v^* - T_{\pi^*}v_{n_k}, \implies 
\norm{v^* - v_{n_{k+1}}} \leq \norm{v^* - T_{\pi^*}v_{n_k}},    \\
\norm{v^* - T_{\pi^*}v_{n_k}} &= \norm{T_{\pi^*}v^* - T_{\pi^*}v_{n_k})} \leq \gamma \norm{v^* - v_{n_k}} \\
\implies \norm{v^* - v_{n_{k+1}}} &\leq \gamma \norm{v^* - v_{n_k}}. 
\end{align}

The probability of the failure event 
\begin{align}
    \Prob(\event^c)  &\leq \sum_{i=1}^{l} \sum_{s \in \cS} \Prob(\event_i^c) \\
    & \leq Sl(1-p_{min})^N
\end{align}
with $p_{min} = \min_{s \in \cS} p(s)$ instead of $q_{min}$.  The rest follows similar reasoning to the proof of \cref{thm:conv_rate} and obtain the result.
\end{proof}

\begin{customcorollary}{3}[Computational complexity of Asynchronous VI]
 Fix an $\epsilon \in (0, \norm{v^* - v_0})$, and assume Asynchronous VI initialises according to (i) or (ii) of \cref{asumpt:init}.  Define 
 \begin{align}
     H_{\gamma, \epsilon} \doteq \ln \left(\frac{\norm{v^* - v_0}}{\epsilon}\right)/{1-\gamma}
 \end{align}
 as a horizon term.  Then, Asynchronous VI returns an $\epsilon$-optimal policy $\pi_n: v_{\pi_n} \geq v^* - \epsilon \1$ with probability at least $1-\delta$ using
 \begin{align}
     O\left( A S H_{\gamma, \epsilon} \left(\ln\left(\frac{S H_{\gamma, \epsilon}}{\delta} \right)/\ln\left( \frac{1}{1-p_{min}}\right)\right) \right)
 \end{align}
 elementary arithmetic and logical operations.  Note that $\norm{v^* - v_0}$ is unknown but it can be upper bounded by $\frac{1}{1-\gamma} + \norm{v_0}$ given rewards are in $[0,1]$.
 \end{customcorollary}
  \begin{proof}
Recall from \cref{lemma:avi_monotone}, the iterates of Asynchronous VI,  $v_n \to v^*$ monotonically from below (i.e., $v_n \geq v_{n-1} \geq \dots \geq v_0$ ). For any $n \in \mathbb{N}_0$ and $s_n \in \cS$, let $\pi_{n+1}(s_n) = \argmax_{a \in \A} L^{v_{n}}(s_n,a)$ with $\pi_{n+1}(\bar s) = \pi_n(\bar s)$ for all other $\bar s \neq s_n$. One can show $L^{v_{n}}(s, \pi_n(s)) \geq v_n(s)$ for any $s \in \cS$ and $n \in \mathbb{N}_0$ following a similar argument as in the proof of \cref{cor:complexity}.  Now, we show $v_{\pi_n} \geq v_n$ for any $n \in \mathbb{N}_0$. Fix $n$ and if we are to apply the policy evaluation operator $T_{\pi_n}$ that satisfy \cref{lemma:fixpoint}(1) to every state $s \in \cS$, then we obtain 
\begin{align}
    T_{\pi_n}v_n(s) = L^{v_n}(s, \pi_n(s)) &\geq v_n(s). 
\end{align}
Therefore, $T_{\pi_n}v_n \geq v_n$.  By applying the $T_{\pi_n}$ operator to $T_{\pi_n} v_n \geq v_n$ repeatedly, and by using the monotonicity of $T_{\pi_n}$, we have for any $k \geq 0$, 
\begin{align}
    T^{k}_{\pi_n}v_n &\geq T^{k-1}_{\pi_n}v_n \geq \cdots \geq v_n.
\end{align}
By taking limits of both sides of $T_{\pi_n}^k v_n \geq v_n$ as $k \to \infty$, we get $v_{\pi_n} \geq v_n$.  Therefore, 
\begin{align}
    0 \leq v^* - v_{\pi_n} \leq v^* - v_n \implies
    \norm{v^* - v_{\pi_n}} \leq \norm{v^* - v_n}.\label{eq:95}
\end{align}

Next, recall from the proof of \cref{thm:avi_conv_rate} that for a given $l \in \mathbb{N}$, and with probability $1 - \delta$, $v_n$ of Asynchronous VI would have $\gamma$-contracted at least $l$ times (i.e., $\norm{v^* - v_n} \leq \gamma^l \norm{v^* - v_0}$) with 
$n \geq l \left \lceil \ln \left( \frac{Sl}{\delta} \right) \Big / \ln \left( \frac{1}{1 - p_{min}}\right) \right \rceil$.  
Following from \cref{eq:95}, with probability $1 - \delta$, 
\begin{align}
    \norm{v^* - v_{\pi_n}} \leq \norm{v^* - v_n} \leq \gamma^l \norm{v^* - v_0}. 
\end{align}
By setting $\gamma^l \norm{v^* - v_0} = \epsilon$ and solve for $l$, we get:
\begin{align}
    l &= \ln{\frac{\norm{v^*-v_0}}{\epsilon}}/\ln\left(\frac{1}{\gamma} \right). 
\end{align}
We observe that $
    \ln{\left(\frac{\norm{v^*-v_0}}{\epsilon}\right)}/\ln\left(\frac{1}{\gamma} \right) \leq \ln{\left(\frac{\norm{v^*-v_0}}{\epsilon}\right)}/(1-\gamma) \doteq H_{\gamma, \epsilon}.$
To compute $v_n$, Asynchronous VI takes $O(AS)$ elementary arithmetic operations.  With probability $1-\delta$, Asynchronous VI obtains an $\epsilon$-optimal policy within
\begin{align}
    O(ASn) &= 
   O\left(AS H_{\gamma, \epsilon} \left(\ln\left( \frac{S H_{\gamma, \epsilon}}{\delta} \right)/\ln\left( \frac{1}{1-p_{min}}\right)\right)\right) \label{eq:o_avi_complex}
\end{align}
arithmetic and logical operations.
\end{proof}

\section{More experiments} \label{appendix:exp}
In this section, we show additional experiments with the MDPs described in \cref{sec:exp} with rewards generated via a standard normal and a Pareto distribution.  

Recall that the experiments were set up to see how DAVI's performance is affected by the sparsity of rewards.    Pareto distribution with a shape of 2.5 is a ``heavy-tail" distribution, and the rewards sampled from this distribution could result in a few large values.    On the other hand, the rewards sampled via the standard Normal distribution could result in many similar values.    We hypothesize that DAVI would converge faster than Asynchronous VI in domains with multiple optimal or near-optimal policies, which could be the case in the normal-distributed reward setting.  

The algorithms that will be compared in the experiments are VI, Asynchronous VI, and DAVI.    We implement Asynchronous VI and DAVI using uniform sampling to obtain the states.    DAVI samples a new set of actions via uniform sampling without replacement in each iteration.

\subsection{Single-state experiment}
This experiment consists of a single-state MDP with 10000 actions, and all terminate immediately.    We experiment with two reward distributions: Pareto-reward and Normal-reward.    For Pareto-reward, all actions have rewards generated according to a Pareto distribution with shape 2.5.    For Normal-reward, all actions have rewards generated according to the standard normal distribution.  

\subsection{Multi-reward experiment}
This experiment consists of two MDPs.    The first set consists of a tree with a depth of 2.    Each state has 50 actions, where each action leads to 2 other distinct next states.    All actions terminate at the leaf states.    In one setting, the rewards are distributed according to the Pareto distribution with a shape of 2.5.  In the other setting, the rewards are distributed according to the normal distribution. 

The second set of MDPs consists of a random MDP with 100 states, where each state has 1000 actions.    Each action leads to 10 next states randomly selected from the 100 states with equal probability.    All transitions have a 0.1 probability of terminating.    In one setting, the rewards are distributed according to the Pareto distribution with a shape of 2.5.  In the other setting, the rewards are distributed according to the standard normal distribution.    The $\gamma$ in all of the MDPs are 1.  

\subsection{Discussion}
\Cref{fig:single_pareto_normal} and \Cref{fig:multi_pareto_normal} show the performance of the algorithms.  All graphs included error bars showing the standard error of the mean.  All graphs smoothly increased towards the asymptote except for Asynchronous VI in \Cref{fig:single_pareto_normal} and VI in \Cref{fig:multi_pareto_normal}, whose performances were step-functions \footnote{Asynchronous VI is equivalent to VI in the single-state experiment since there is only one state.}.    The y-axis of each graph showed a state value averaged over 200 runs.    The x-axes showed run-times, which have been adjusted for computations.

In \Cref{fig:single_pareto_normal}, DAVI with $m=1$ was significantly different from that of DAVI with $m = 10, 100, 1000$.  However, in the Normal-reward setting, the performance of DAVI with $m=1$ was much closer to the performance of DAVI with $m=10, 100, 1000$.    In the Pareto-reward setting, where there could only be a few large rewards, the results were similar to that of the needle-in-the-haystack setting of \Cref{fig:single}.   In the Normal-reward setting, where most of the rewards were similar and concentrated around 0, the results were similar to that of the multi-reward setting of \Cref{fig:single}.   
\begin{figure}[!ht]
    \centering
    \includegraphics[keepaspectratio=true,scale=0.43]{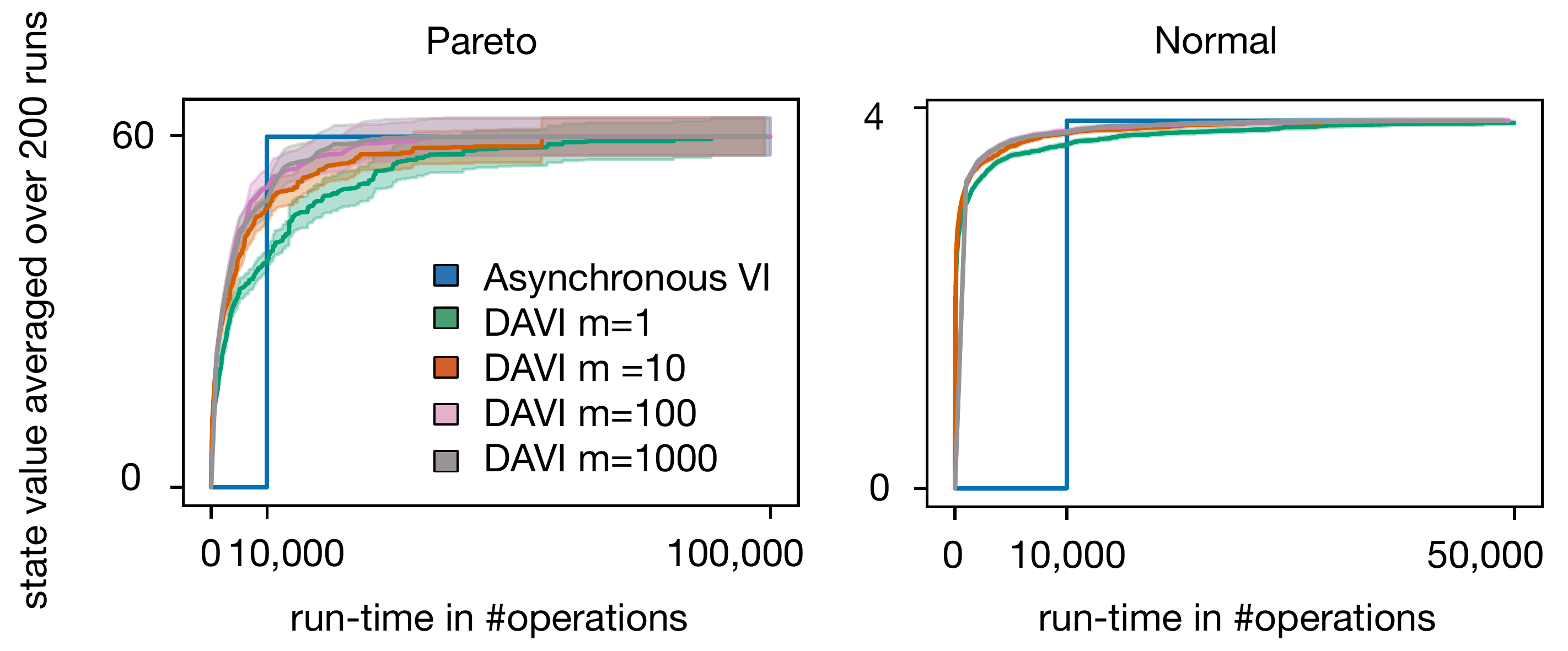}
    \caption{Single-state experiment with $10000$ actions: (a) rewards are Pareto distributed with shape 2.5 (b) rewards are standard normal distributed.  The Asynchronous VI in this experiment is equivalent to VI since there is only one state. We run each instance 200 times with a new MDP generated each time. In each run, each algorithm is initialized to 0.}
    \label{fig:single_pareto_normal}
\end{figure}

In \Cref{fig:multi_pareto_normal} in both tree Pareto-reward and Normal-reward settings (top row), DAVI with $m = 1$ was significantly different from that of DAVI $m=10$.  In the tree setting, with normal-distributed rewards, where there may be multiple actions with similarly large rewards, DAVI $m=10$ converged faster than VI and Asynchronous VI.
\begin{figure}[!ht]
    \centering
    \includegraphics[keepaspectratio=true,scale=0.4]{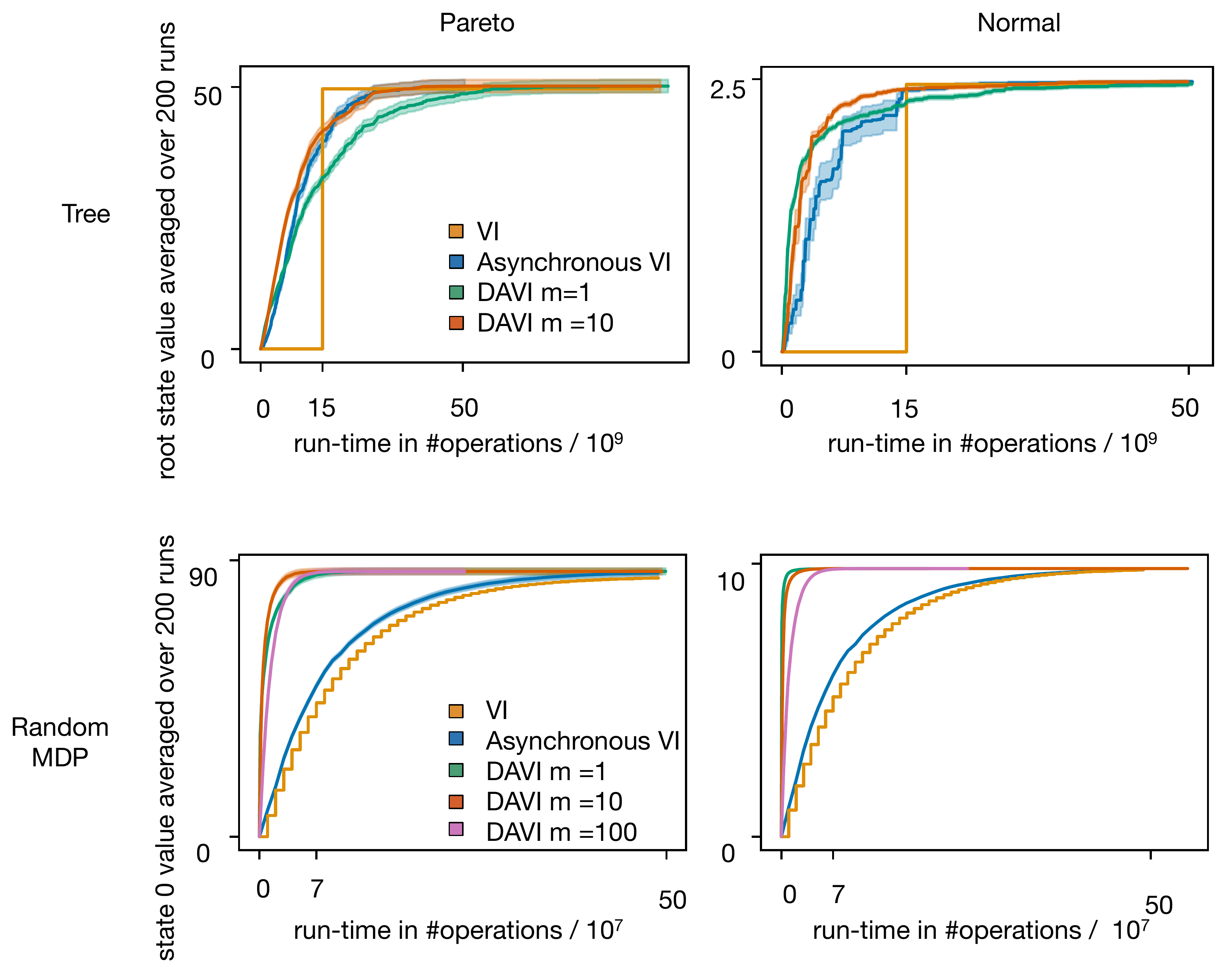}
    \caption{Multi-state experiment: (top row) MDP with a tree structure with Pareto and normal distributed rewards (bottom row) random MDP with Pareto and normal distributed rewards.  We run each instance 200 times with a new MDP generated each time. In each run, all algorithms are initialized to 0. }
    \label{fig:multi_pareto_normal}
\end{figure}

In the random-MDP setting, DAVI, for all values of $m$, converged faster than VI and Asynchronous VI in both the Pareto-reward and Normal-reward settings, as evident in the bottom row of \Cref{fig:multi_pareto_normal}.    As expected, DAVI converged faster than Asynchronous VI and VI in the case of multiple near-optimal policies.    Note, DAVI $m = 100$ was the slowest to converge, a case where the action subset size is large.   This result makes sense as Asynchronous VI with the full action space did not converge as fast as DAVI with smaller action subsets.

\end{document}